%% file: main.tex
\documentclass{article}

\usepackage[utf8]{inputenc}
\usepackage[colorlinks = true,
            linkcolor = blue,
            urlcolor  = blue,
            citecolor = blue,
            anchorcolor = blue]{hyperref}
\usepackage{amsmath}
\usepackage{amsthm}
\usepackage{xcolor}
\usepackage{algorithm}
\usepackage{algorithmic}
\usepackage[capitalise]{cleveref}
\usepackage{graphicx}
\usepackage{amssymb}
\usepackage{verbatim}
\usepackage{nicefrac}
\usepackage{cite}
\usepackage{geometry}
\makeatletter
\theoremstyle{plain}
\newtheorem{thm}{\protect\theoremname}
\theoremstyle{definition}
\newtheorem{defn}[thm]{\protect\definitionname}

\newtheorem{example}[thm]{\protect\examplename}
\theoremstyle{plain}
\newtheorem{lem}[thm]{\protect\lemmaname}

\newtheorem{prop}[thm]{\protect\propname}

\newtheorem{conj}[thm]{\protect\conjecturename}

\makeatother

\providecommand{\corollaryname}{Corollary}
\providecommand{\definitionname}{Definition}
\providecommand{\lemmaname}{Lemma}
\providecommand{\theoremname}{Theorem}
\providecommand{\exercisename}{Exercise}
\providecommand{\examplename}{Example}
\providecommand{\remarkname}{Remark}
\providecommand{\propname}{Proposition}
\providecommand{\conjecturename}{Conjecture}
\usepackage{soul}
\author{Alon Gonen\footnote{University of California San Diego} \and Shachar Lovett\footnote{University of California San Diego} \and Michal Moshkovitz\footnote{University of California San Diego}}
\title{Towards a combinatorial characterization of bounded memory learning}
\begin{document}
\maketitle 

\input{abstract.tex}

\input{intro.tex}

\input{pre.tex}

\input{sqToBM.tex}

\input{BMtoSQ.tex}

\newpage
\bibliographystyle{plain}
\bibliography{library}

\newpage

\appendix

\input{appendix.tex}
\end{document}

%% file: abstract.tex
\begin{abstract}
Combinatorial dimensions play an important role in the theory of machine learning. For example, VC dimension characterizes PAC learning, SQ dimension characterizes weak learning with statistical queries, and Littlestone dimension characterizes online learning.

In this paper we aim to develop combinatorial dimensions that characterize bounded memory learning. We propose a candidate solution for the case of realizable strong learning under a known distribution, based on the SQ dimension of neighboring distributions. We prove both upper and lower bounds for our candidate solution, that match in some regime of parameters. In this parameter regime there is an equivalence between bounded memory and SQ learning.  
 We conjecture that our characterization holds in a much wider regime of parameters.
\end{abstract}

%\begin{abstract}
%The fundamental theorem of statistical learning exactly characterize unconditional learnability of a class $\mathcal{C}$, i.e., it provides a necessary and sufficient combinatorial condition, $VC(C)\neq\infty$. Unfortunately, finding a characterization of learnability under memory constraints is still lacking. 
%In this paper we provide a candidate for such a characterization, assuming the examples are from some fixed known distribution $P$. Namely, we suggest that a class is learnable if and only if neighboring distributions has small SQ dimension. We also discuss the connection of this combinatorial condition to SQ learnability.  
%\end{abstract}

%% file: intro.tex
%\scomment{general formatting comments:\\
%1. why use both propositions and lemmas? is one for small claims, one for big claims? I typically see claims used for small claims, lemmas for bigger ones.\textcolor{blue}{Michal: proposition are for claims in other works, the rest are lemmas}\\
%2. switch to using cleveref for references. \textcolor{blue}{done}\\
%3. use macros for $\mathcal{C}, \mathcal{X}$, SQ dimension, and other common mathematical symbols we use, so it will be easy to change them later if needed } \textcolor{blue}{Michal: Alon prefers it like that, and I think it's not that important at this step}

\section{Introduction}
Characterization of different learning tasks using a combinatorial condition has been investigated in depth in machine learning. Learning a class in an unconstrained fashion is characterized by a finite VC dimension \cite{vapnik15,blumer89}, %\scomment{add citation},  
and weakly learning in the statistical query (SQ) framework is characterized by a small SQ dimension \cite{blum94}.
%\scomment{add citation}. 
Is there a simple combinatorial condition that characterizes learnability with bounded memory? In this paper we propose a candidate condition,  prove upper and lower bounds that match in some of the regime of parameters, and conjecture that they match in a much wider regime of parameters.

A learning algorithm that uses $b$ bits of memory, $m$ samples, and accuracy $1-\epsilon$ is defined as follows: the algorithm receives a series of $m$ labeled examples one by one, while only preserving an internal state in $\{0,1\}^b$ between examples. In this paper we focus our attention on the realizable setting: the labeled examples are pairs $(x_i, c(x_i))$, where $x_i \in \mathcal{X}$ and $c:\mathcal{X} \to \{-1,1\}$ is a concept in a concept class $\mathcal{C}$. The algorithm is supposed to return with constant probability a hypothesis $h$ which matches the unknown concept $c$ on a $1-\epsilon$ fraction of the underlying distribution. In this paper we further assume that the underlying distribution $P$ on $\mathcal{X}$ is known to the learner, similar to the setting in the SQ framework.

There are two ``trivial'' algorithms for the problem which we now present. For ease of presentation, we restrict our attention in the introduction to a small constant $\epsilon$, say $\epsilon=0.01$. Without making any additional assumptions, the following space complexity bounds are known when learning with accuracy $0.99$:
\begin{enumerate}
\item The ERM algorithm keeps in memory   $m=O\left(\log|\mathcal{C}|\right)$ samples, and outputs a hypothesis that is consistent with the entire sample. This requires $b=O\left(\log|\mathcal{C}|\log|\mathcal{X}|\right)$ bits.
\item
A learning algorithm that
enumerates all possible concepts in $\mathcal{C}$ and the consistency of each concept based on few random samples. This algorithms requires $m=O(|\mathcal{C}|\log|\mathcal{C}|)$ samples and $b=O(\log|\mathcal{C}|)$ bits. 
\end{enumerate}

We define a class $\mathcal{C}$ under a distribution $P$ to be \emph{learnable with bounded memory} if there is a ``non-trivial'' learning algorithm with respect to both sample complexity and space complexity. A bit more formally, if there is a learning algorithm that uses only $m=|\mathcal{C}|^{o(1)}$ samples and
 $b=o(\log|\mathcal{C}|\log|\mathcal{X}|)$ bits (see \cref{dfn:bounded_memory_learning}).
 
To introduce our main result we need two definitions.
The statistical query (SQ) dimension, $SQ_P(\mathcal{C})$, is a known complexity measure that characterizes weak learning in the SQ framework (see \cref{def:SQdim}). Given a distribution $P$, we say that a distribution $Q$ is $\mu$-close to it (where $\mu \ge 1$) if the ratio $P(x)/Q(x)$ is between $1/\mu$ and $\mu$ for all points $x$ in the domain. We denote by $\mathcal{P}_{\mu}(P)$ the set of all distributions which are $\mu$-close to $P$ (see \cref{def:muclose}).

Our main results are upper and lower bounds on bounded memory learning, in terms of the SQ dimension of distributions in the neighbourhood of the underlying distribution $P$:

\begin{enumerate}
    \item Suppose that there is a parameter $d \ge 1$ such that for any distribution $Q \in\mathcal{P}_{d}(P)$ it holds that $\mathrm{SQ}_{Q}(\mathcal{C})\le d$. Then there exists an algorithm that learns the class $\mathcal{C}$ with accuracy $0.99$ under the distribution $P$ using $b=O(\log (d) \cdot\log\vert\mathcal{C}\vert)$
bits and $m=\mathrm{poly}(d) \cdot \log(|\mathcal{C}|)\cdot\log\log(|\mathcal{C}|)$ samples.
\item If the class $\mathcal{C}$ is PAC-learnable under $P$ with accuracy $0.99$ using $b$ bits and $m$ samples, then for every distribution $Q \in\mathcal{P}_{\Theta(1)}(P)$
its SQ dimension is bounded by  $SQ_{Q}(\mathcal{C})\leq\max(\mathrm{poly}(m),2^{O(\sqrt{b})})$. 
\end{enumerate}
In \cref{sec:main_results} we give a more detailed account of the bounds for general $\epsilon$. We show that for small enough $\epsilon$, the two conditions coincide and we in fact get a characterization of bounded memory learning. We conjecture that the characterization holds for a larger range of parameters (see \cref{conj:main}).
We also prove similar conditions for SQ learning, thus implying equivalence between bounded memory learning and SQ learning for small enough $\epsilon$.

\subsection{Problem setting}
In this paper we consider two learning frameworks: a) The PAC model \cite{valiant84} and b) The Statistical Query framework \cite{kearns98}.
\paragraph{PAC model.}
%\scomment{need to cite Valiant here as well.}
In PAC learning \cite{valiant84} we consider the task of binary classification over an \textit{instance space} $\mathcal{X}$. 
Denote by $ \mathcal{C} \subseteq \{-1,1\}^ \mathcal{X} $ a concept class of functions mapping instances to binary labels, and let $c \in \mathcal{C}$ be the \textit{target} (a.k.a. true) concept. Also, let $P$ be the underlying probability distribution over $\mathcal{X}$. We assume that $P$ is known to the learner whereas the target concept $c$ is not known.
%\acomment{do we want to justify the assumption that $P$ is known?} \scomment{we should. it is not standard. we should also emphasize that we are in the realizable setting.}

The input to the learning algorithm $\mathcal{A}$ consists of a labeled sample $S=((x_1,c(x_1)),\ldots,(x_m,c(x_m)))$ such that $S_X:=(x_1,\ldots,x_m) \sim P^{m}$. Its output has the form of a hypothesis $h\in\{-1,1\}^{\mathcal{X}}$. We measure the success of the algorithm according to its expected error
$L_{P,c}(h)=\Pr_{x\sim P}(h(x)\neq c(x))$. We say that $h$ is $\epsilon$-accurate if $L_{P,c}(h)\le \epsilon$. The \textit{sample complexity} of $\mathcal{A}$ under the distribution $P$, denoted $m(\epsilon):(0,1) \rightarrow\mathbb{N}$, is a function mapping a desired accuracy $\epsilon$ to the minimal positive integer $m(\epsilon)$ such that for any target concept $c \in \mathcal{C}$ and any $m \ge m(\epsilon)$, with probability at least $2/3$ over the drawn of an i.i.d. sample $S=((x_1,c(x_1)),\ldots,(x_m,c(x_m)))$, the output $\mathcal{A}(S)$ is $\epsilon$-accurate.\footnote{Given a confidence parameter $\delta > 2/3$, standard amplification techniques can be used to ensure that the probability error is at most $\delta$, while increasing the sample complexity by at most a $\log(1/\delta)$ multiplicative factor.} 

\paragraph{The statistical query framework.} 
The statistical query (SQ) framework has been introduced by  \cite{kearns98} to handle random noise in the PAC setting. In this model, instead of having access to an i.i.d. sequence of labeled instances, the learner has access to a \textit{statistical query oracle} (a.k.a. \textit{correlation oracle}). Each call to the oracle has the form of a pair $(h,\tau)$, where $h \in \{-1,1\}^\mathcal{X}$ is a hypothesis and $\tau > 0$ is called a \textit{tolerance} parameter.
The oracle has to answer such a query with a scalar $\nu$ satisfying\footnote{According to the original framework of Kearns, (seemingly) more general queries are allowed. Namely, each query is a pair $(\chi,\tau)$ where $\chi:\mathcal{X} \times \{-1,1\} \rightarrow \{-1,1\}$. The oracle has to answer the query with a scalar $\nu$ satisfying
\[
|\mathbb{E}_{x \sim P}[\chi(x,c(x))] - \nu| \le \tau~.
\]
%\acomment{do we want to elaborate.}\scomment{it is a footnote. why not?}
Note that $\chi(x,c(x))$ can be written as a polynomial in $x$ and $c(x)$, and since $c(x)$ is either $1$ or $-1$, this polynomial is linear in $c(x)$. In other words, $\chi(x,c(x))=g_1(x)c(x)+g_2(x)$. given that the distribution $P$ is known, $\mathbb{E}_{x \sim P}[g_2(x)]$ can be calculated. Thus, one can simulate the seemingly more general query $\chi$ using the correlation query applied to $g_1$.
%It is not hard to see that given that the distribution $P$ is known, one can simulate the seemingly more general query using the correlation query.
}  
\[
|\langle h,c \rangle_P - \nu| \le \tau \quad \textrm{where} \quad \langle h,c \rangle_P :=\mathbb{E}_{x \sim P} [h(x)c(x)].
\]
As was shown in \cite{kearns98}, any approximately accurate algorithm in the SQ model can be efficiently transformed into an approximately accurate PAC algorithm, i.e. an algorithm that has access to i.i.d. labeled examples. The resulted PAC is also robust to noise. We refer to \cite{szorenyi09} for additional background. 

Analogously to the definition of sample complexity, the \textit{query complexity} of a learning algorithm in the SQ model, denoted $q_\tau(\epsilon)$, is the minimal number of queries with tolerance parameter $\tau$ required for achieving $\epsilon$-accurate prediction (for any target concept $c \in \mathcal{C}$).

\paragraph{SQ dimension.}
The SQ-dimension defined
below is useful for characterizing weak learnability in the statistical
query framework, as was proved in \cite{blum94} (see \cref{prop:weakSQ} and \cref{prop:SQ_lower_bound}). %\scomment{we probably need to cite something here. Where was this first defined?}\textcolor{blue}{Michal:\cite{blum94}}
\begin{defn}[Statistical query dimension]\label{def:SQdim} 
Fix a probability distribution $P$ over $\mathcal{X}$.
The SQ-dimension of the class $\mathcal{C}$ with respect to the distribution $P$, denoted $\mathrm{SQ}_{P}(\mathcal{C})$,
is the maximal integer $d$ such that there exist $h_{1},\ldots,h_{d} \in \mathcal{C}$
satisfying $|\langle h_{i},h_{j}\rangle_{P}| \le1/d$
for all $i\neq j \in [d]$. 
\end{defn}

%\scomment{we need to explain how SQ dimension relates to SQ learnability, and cite the relevant works.}

\paragraph{Bounded memory learning.}
A bounded memory learning algorithm observes a sequence of labeled examples $(x_1,y_1),(x_2,y_2),\ldots$ in a streaming fashion, where $x_i \in \mathcal{X}, y_i \in \{-1,1\}$. We assume in this paper that the data is realizable, namely $y_i=c(x_i)$ for some concept $c \in \mathcal{C}$.
The algorithm maintains a state $Z_t \in \{0,1\}^b$ after seeing the first $t$ examples, and update it after seeing the next example to $Z_{t+1} = \psi_t(Z_t,(x_{t+1},y_{t+1}))$ using some update function $\psi_t$.\footnote{Following the model of branching programs (e.g., \cite{raz16}), the maps $\psi_1,\psi_2,\ldots$ are not considered towards the space complexity of the algorithm.}   The parameter $b$ is called the \textit{bit complexity} of the algorithm. Finally, after observing $m$ samples (where $m$ is a parameter tuned by the algorithm), a hypothesis $h=\phi(Z_m)$ is returned. 

%\acomment{we should explain these observations. Maybe we can do it in the into while describing the main case study.} \scomment{now is good.} 
We now expand the two ``trivial'' learning algorithms described earlier to accuracy $1-\epsilon$ for any $\epsilon>0$:
\begin{enumerate}
\item 
We can learn with accuracy $1-\epsilon$ using $m=O\left(\log|\mathcal{C}| \mathrm{poly}(1/\epsilon)\right)$ samples and number of bits equal to $b=O\left(\log|\mathcal{C}|\log|\mathcal{X}|+\log|\mathcal{C}| \log(1/\epsilon)\right)$. For constant accuracy parameter this can be done by saving $O(\log|C|)$ examples and applying ERM. To achieve better accuracy we can apply Boosting-By-Majority \cite{freund95} as we describe in \cref{sec:bounded_sq_implies_bounded_memory}. 
\item
One can always learn with $m=O(|\mathcal{C}|\log|\mathcal{C}|\epsilon^{-1})$ samples and $b=O(\log|\mathcal{C}|)$ bits, by going over all possible hypothesis and testing if the current hypothesis is accurate on a few random samples. 
\end{enumerate}

We define a class $\mathcal{C}$ to be bounded memory learnable if there is a learning algorithm that beats both of the above learning algorithms. 

\begin{defn}[Bounded memory learnable classes]
\label{dfn:bounded_memory_learning}
A class $\mathcal{C}$ under a distribution $P$ is \emph{learnable with bounded memory} with accuracy $1-\epsilon$ if there is a learning algorithm that uses only $m=\left(|\mathcal{C}|/\epsilon\right)^{o(1)}$ samples and
 $b=o(\log|\mathcal{C}|(\log|\mathcal{X}|+\log(1/\epsilon)))$ bits\footnote{Formally, the $o(\cdot)$ factors are in terms of the size of the class $\mathcal{C}$. Hence this definition applies to families of distributions $\{\mathcal{C}_n\}$ of growing size, for example parities on $n$ bits. However, in the main theorems we give quantitative bounds and hence can focus on single classes instead of families of classes.}. 
\end{defn}

To illustrate this, consider the case where the number of concepts and points are polynomially related, $|\mathcal{C}|,|\mathcal{X}|=\text{poly}(N)$, and where the desired error in not too tiny, $\epsilon \ge 1/\text{poly}(N)$. Then a non-trivial learning algorithm is one that uses a sub-polynomial number of samples $m=N^{o(1)}$ and a sub-quadratic number of bits $b=o(\log^2 N)$. There are classes that can not be learned with bounded memory.

\begin{example}[Learning parities]
Consider the task of learning parities on $n$ bits. Concretely, let $N=2^n$,  $\mathcal{X}=\mathcal{C}=\{0,1\}^n$, $P$ be the uniform distribution over $\mathcal{X}$, and let the label associated with a concept $c \in \mathcal{C}$ and point $x \in \mathcal{X}$ be $\langle c,x \rangle~(\mathrm{mod}~2)$. It was shown by \cite{raz16,moshkovitz18} that achieving constant accuracy for this task requires either $b=\Omega(n^2)=\Omega(\log^2 N)$ bits of memory or an exponential in $n$ many samples, namely $m=2^{\Omega(n)}=N^{\Omega(1)}$ samples.
\end{example}

\paragraph{Close distributions.}
An important ingredient in this work is the notion of nearby distributions, where the distance is measured by the multiplicative gap between the probabilities of elements.

\begin{defn}[$\mu$-close distributions]
\label{def:muclose}
We say that two distributions $P,Q$ on $\mathcal{X}$ are $\mu$-close for some $\mu\geq 1$ if 
$\mu^{-1} P(x) \le Q(x)\le \mu P(x)$ for all $x\in\mathcal{X}$. Note that the definition is symmetric with respect to $P,Q$.
We denote the set of all distributions that are $\mu$-close to $P$ by $\mathcal{P}_\mu(P)$. 
% If $\mu$ is a constant we simply write $P \in \mathcal{P}(P)$.
\end{defn}

\subsection{Main results}\label{sec:main_results}
\paragraph{Bounded memory PAC learning.}

We state our main results for a combinatorial characterization of bounded memory PAC learning in terms of the SQ dimension of distributions close to the underlying distribution.

\begin{thm} \label{thm:sq2bmPAC}
Let $\epsilon \in (0,1)$, $d \in \mathbb{N}$ and denote by $\mu = \Theta (\max \{d, 1/\epsilon^3 \})$. Suppose that the distribution $P$ satisfies the following condition: for any distribution $Q \in\mathcal{P}_{\mu}(P)$, $\mathrm{SQ}_{Q}(\mathcal{C})\le d$. Then there exists an algorithm that learns the class $\mathcal{C}$ with accuracy $1-\epsilon$ under the distribution $P$ using $b=O(\log (d/\epsilon) \cdot\log\vert\mathcal{C}\vert)$
bits and $m=\mathrm{poly}(d/\epsilon) \cdot \log(|\mathcal{C}|) \cdot \log\log(|\mathcal{C}|)$ samples. 
\end{thm}

\begin{thm}\label{thm:bm2sqPAC}
If a class $\mathcal{C}$ is strongly PAC-learnable under $P$ with accuracy $1-0.1\epsilon$ using $b$ bits and $m$ samples, then for every distribution $Q \in \mathcal{P}_{1/\epsilon}(P)$, its SQ-dimension is bounded by  $SQ_{Q}(\mathcal{C})\leq\max\left(\mathrm{poly}(m/\epsilon),2^{O(\sqrt{b})}\right)$. 
\end{thm}

There is a regime of parameters where the upper and lower bounds match. Let $|\mathcal{C}|,|\mathcal{X}|=\text{poly}(N)$ and that $\epsilon=N^{-o(1)}$. Recall that the class is bounded memory learnable if there is a learning algorithm with sample complexity $m=N^{o(1)}$ and space complexity $b=o(\log^2 N)$. Let $\mu,d=N^{o(1)}$. We have the following equivalence:
%Focusing on the case that $|\mathcal{C}|=\Theta(|\mathcal{X}|)$, $\mu=d=poly(1/\epsilon)=|C|^{o(1)}$, $m=|\mathcal{C}|^{o(1)}$, $b=o(\log|\mathcal{C}| \cdot \log(|\mathcal{X}|/\epsilon))$, we get a characterization of bounded memory learning:  
\textit{\begin{center}
$\mathcal{C}$ is bounded memory learnable under $P$ with accuracy $1-\epsilon$\\ %$\Longleftrightarrow$\\ 
$\Updownarrow$\\ 
$\forall Q\in\mathcal{P}_{\textrm{poly}(1/\epsilon)}(P)$, $\mathrm{SQ}_{Q}(\mathcal{C})\le poly(1/\epsilon)\;$.
\end{center}}

%\scomment{now is a good time to repeat the examples from the beginning, and show how they fit the parameters in the theorems above.} \textcolor{blue}{Michal: clear enough?}

We conjecture that this equivalence holds for any $\epsilon$.

\begin{conj}\label{conj:main}
For any $\epsilon$, the class $\mathcal{C}$ is bounded memory learnable under distribution $P$ with accuracy $1-\epsilon$ $\Longleftrightarrow$ $\forall Q\in\mathcal{P}_{\textrm{poly}(1/\epsilon)}(P)$, $\mathrm{SQ}_{Q}(\mathcal{C})\le poly(1/\epsilon)$.
\end{conj}

\paragraph{SQ learning.}
Next, we give our secondary results for SQ learning, which are very similar to our results for bounded memory learning. Conceptually, it shows that the two notions are tightly connected. 

\begin{thm} \label{thm:sq2bmSQ}
Let $\epsilon \in (0,1)$, $d \in \mathbb{N}$ and denote by $\mu = \Theta (\max \{d, 1/\epsilon^3\})$. Suppose that the distribution $P$ satisfies the following condition: for any distribution $Q \in\mathcal{P}_{\mu}(P)$, $\mathrm{SQ}_{Q}(\mathcal{C})\le d$. Then there exists an SQ-learner that learns the class $\mathcal{C}$ with accuracy $1-\epsilon$ under the distribution $P$ using $q=\mathrm{poly}(d/\epsilon)$ statistical queries with tolerance  $\tau \ge \mathrm{poly}(\epsilon/d)$.
\end{thm}

\begin{thm} \label{thm:bm2sqSQ}
If a class $\mathcal{C}$ is strongly SQ-learnable under $P$ with accuracy $1-0.1\epsilon$, $q$ statistical queries, and tolerance $\tau$, then for every distribution $Q \in \mathcal{P}_{1/\epsilon}(P)$,  $SQ_{Q}(\mathcal{C})\leq\mathrm{poly}(q/\epsilon\tau)$. 
\end{thm}

Note that for any class $\mathcal{C}$, underlying distribution and accuracy $1-\epsilon$,
one can SQ-learn the class with $q=|\mathcal{C}|$ statistical queries and tolerance $\tau=O(\epsilon)$, by going over all the hypotheses. Thus a class is non-trivially SQ-learnable if one can learn it with $q=|C|^{o(1)}$ queries and tolerance $\tau\geq \mathrm{poly}(\epsilon)$. 
Focusing on the case that $|\mathcal{C}|,|\mathcal{X}|=\text{poly}(N)$ and $\mu,d,q,1/\epsilon,1/\tau=N^{o(1)}$, we get that bounded memory learning is equivalent to SQ learning.

\subsection{Related work} %\textcolor{blue}{Michal: rewrite}
    \paragraph{Characterization of bounded memory learning.}% Some of these works provide a necessary condition for  
    Many works have proved lower bounds under memory constraints \cite{shamir14,raz16,kol17,moshkovitz17,moshkovitz18,raz17,garg18,dagan18,beame18,sharan19,garg19,dagan19}. Some of these works even provide a necessary condition for learnability with bounded memory. As for upper bounds, not many works have tried to give a general property that implies learnability under memory constraints. One work suggested such property   \cite{moshkovitz17general} but this did not lead to a full characterization of bounded memory learning. 
    
    \paragraph{Statistical query learning.} After Kearns's introduction of statistical query \cite{kearns98}, Blum et al. \cite{blum94} characterized weak learnability using SQ dimension. Specifically, if $SQ_P(\mathcal{C})=d$, then $\mathrm{poly}(d)$ queries are both needed and sufficient to learn with accuracy $1/2+\mathrm{poly}(1/d)$. Note that the advantage is very small, only $\mathrm{poly}(1/d)$. Subsequently several works  \cite{balcazar07,simon07,szorenyi09, feldman12} suggested a few characterizations of strong SQ learnability.

     \paragraph{Bounded memory and SQ dimension.} In this paper we prove an equivalence, in some parameters regime, between bounded memory learning and SQ learning. There were a few indications in the literature that such an equivalence exists. The work \cite{steinhardt16} showed a general reduction from any SQ learner to a memory efficient learner. Alas, they gave an example that suggests that an equivalence is incorrect, which we now address.
     
      \begin{example}[Learning sparse parity]
      Consider the concept class of parity on the first $k$ bits of an $n$-bit input for $k \ll n$, for example $k=\sqrt{n}$. That is, $\mathcal{X}=\{0,1\}^n$ and $\mathcal{C}=\{0,1\}^k \cdot \{0\}^{n-k}$ is a subset of all possible parities. 
       Naively, an ERM algorithm would need to store $\Theta(k)$ examples, each requiring $n$ bits, and hence need $b=\Theta(k n)$. However, it suffices to store only the first $k$ bits of each example, and hence only use $b=\Theta(k^2)$ bits. 
       As this is significantly less than the naive bound of $\Theta(kn)$ we consider the class to be bounded memory learnable. On the other hand, the SQ dimension of $\mathcal{C}$ is maximal, namely $2^k$, and hence \cite{steinhardt16} suggest that this example separates bounded memory learning and SQ learning. 
       
       Relating to our results, it shows two things: when the sizes of the concept classes $\mathcal{C}$ and example set $\mathcal{X}$ are polynomially related, there is no such separation (we prove this for small enough $\epsilon$ and conjecture for all $\epsilon$). Moreover, the $2^{O(\sqrt{b})}$ term in \Cref{thm:bm2sqPAC} is tight. 
      \end{example}

     The work \cite{garg18} showed that high SQ dimension implies non-learnability with bounded memory when the learner returns the exact answer. However, learnability is usually inexact and this does not relate to strong learnability.   
    
    \paragraph{Littlestone dimension.} Online learnability without memory constraints is characterized using Littlestone dimension \cite{littlestone88}. This dimension is not suited for bounded memory learning as it does not take into account the structure of the class which determines whether the class is learnable with bounded memory or not. Specifically, there are problems that have similar Littlestone dimension
    (e.g., parity and discrete thresholds on the line), where the former (thresholds) is easy to learn under memory constraints and the latter (parity) is hard.
    
    \paragraph{Learning under a known distribution.} In SQ framework most works focused on learning under known distributions \cite{blum94,bshouty02,yang01,yang05,balcazar07,simon07,feldman12,szorenyi09}. However, PAC learning research under known distribution is scarce but exists, e.g., \cite{benedek91,ben08,vayatis99,sabato13}. In particular, Benedek {et al.} \cite{benedek91}  showed that unconstrained learning under known distribution is characterized by covering.
    
    \paragraph{Smooth distributions.} A key idea in this paper is to use \emph{close} distributions which are upper and lower bounded by a distribution. A one sided closeness, namely the upper bound, is referred in the literature as a \textit{smooth} distribution, see for example \cite{bshouty02}. 
    Smooth distributions were also used to show equivalence between boosting and hard-core sets \cite{klivans99, impagliazzo95}.
    %They show how to transform each boosting algorithm to smooth. Cite Russell.

\subsection{Paper organization}
 We begin in \cref{sec:prem} with a presentation of known results in boosting and statistical queries that we will need. In \cref{sec:bounded_sq_implies_bounded_memory} we construct learning algorithms based on the assumption that close distributions have bounded SQ dimensions, and prove  \cref{thm:sq2bmPAC} and \cref{thm:sq2bmSQ}.
 In \cref{sec:bounded_memory_implies_bounded_sq} we establish the reverse direction and prove \cref{thm:bm2sqPAC} and \cref{thm:bm2sqSQ}. Omitted proofs can be found in the appendix.

%% file: pre.tex
\section{Preliminaries}\label{sec:prem}

We review some well known definitions and results necessary for our work. 

\paragraph{Weak learning and boosting.} It is often conceptually easier to design an algorithm whose accuracy is slightly better than an educated guess, and then attempt to boost its accuracy. 

Consider first the PAC model. We say that a learning algorithm $\mathcal{W}$ is a $\gamma$-\textit{weak learner} if there exists an integer $m$ such that for any target concept $c \in \mathcal{C}$ and any $n \ge m$, with probability at least $2/3$ over the draw of an i.i.d. labeled sample $S = ((x_1,c(x_1)),\ldots,(x_n,c(x_n)))$ according to the underlying distribution $P$, the hypothesis returned by $\mathcal{A}$ is $(1/2-\gamma)$-accurate. We refer to the minimal integer $m$ satisfying the above as the \emph{sample complexity} of the weak learner. The notion of $\gamma$-weak learning in the SQ framework is defined analogously, where the \textit{query complexity} of the weak learner is denoted by $q_{\tau}$ (where $\tau$ is the tolerance parameter).

A \textit{boosting} algorithm $\mathcal{A}$ uses an oracle access to a weak learner $\mathcal{W}$ and aggregates the predictions of $\mathcal{W}$ into a satisfactory accurate solution. The celebrated works of Freund and Schapire \cite{schapire90,schapire91,freund92,freund95,freund97} provide several successful boosting algorithms for the PAC model. The work of \cite{aslam93} extended some of these results to the SQ framework.

\paragraph{Known SQ-dimension bounds for weak learning.}
The following upper and lower bounds are known. The first upper bound is a folklore lemma whose proof can be found in \cite{szorenyi09}.
\begin{prop} \label{prop:weakSQ}
Let $\mathcal{C}$ be a concept class, $P$ an underlying distribution, such that $\mathrm{SQ}_P(\mathcal{C})\leq d$. Then there is a $(1/d)$-weak SQ-learner with query complexity $q=d$ and tolerance $\tau=1/3d$.
\end{prop}
The next lower bound was initially proved by
\cite{Blum2015}. A simplified proof was given later by \cite{szorenyi09}.
\begin{prop}\label{prop:SQ_lower_bound}
Let $\mathcal{C}$ be a concept class, $P$ an underlying distribution, and let $d=SQ_P(\mathcal{C})$. Any learning algorithm that uses tolerance parameter lower bounded by $\tau>0$ requires in the worst case at least $(d\tau^2-1)/2$ queries for learning $\mathcal{C}$ with accuracy at least $1/2+1/d$.
%\scomment{so if $\tau$ is really small, there are no lower bounds?} \textcolor{blue}{Michal: yes. take the extreme case that $\tau=0$, then we can order all examples and give them probability $1/2^i$, one query will give the answer to all the examples}
\end{prop}
Finally, the next proposition shows that SQ learnability (weak or strong) implies learning with bounded memory. 
%\scomment{is this also for weak learning? or for strong learning? it is not clear.}
%\textcolor{blue}{Michal: both. They suggest a way to simulate any SQ algorithm (weak or strong) using few bits}
\begin{prop}[Theorem 7 in \cite{steinhardt16}] \label{prop:steinhardt} Assume that a class $\mathcal{C}$
can be learned using $q$ statistical queries with tolerance $\tau$.
Then there is an algorithm that learns $\mathcal{C}$ using
$m=O(\frac{q\log|\mathcal{C}|}{\tau^{2}}(\log(q)+\log\log(|\mathcal{C}|))$
samples and $b=O(\log|\mathcal{C}|\cdot\log(q/\tau))$ bits.
\end{prop}

\paragraph{Additional notation}
We denote the density and the cumulative binomial distribution by $\mathrm{Binom}(m,r,p)$ and $\mathrm{Binom}(m,\le r,p)$, which respectively
refer to the probability of observing exactly (at most) $r$ heads
in $m$ independent and identical trials where the probability of
``head'' in each single trial is $p$.\footnote{If $r > m$ or $r<0$ then both terms are equal to zero.}

%\subsection{Statistical Query Algorithm}

%% file: sqToBM.tex
\section{From bounded SQ dimension to bounded memory learning}\label{sec:bounded_sq_implies_bounded_memory}
In this section we prove our upper bounds:  \cref{thm:sq2bmPAC} and \cref{thm:sq2bmSQ}. A schematic illustration of the proof is given in  \cref{fig:sq2bm}.
\begin{figure}%[H]
%\caption{Proof Schematic}
\begin{center}
\includegraphics[scale=0.5]{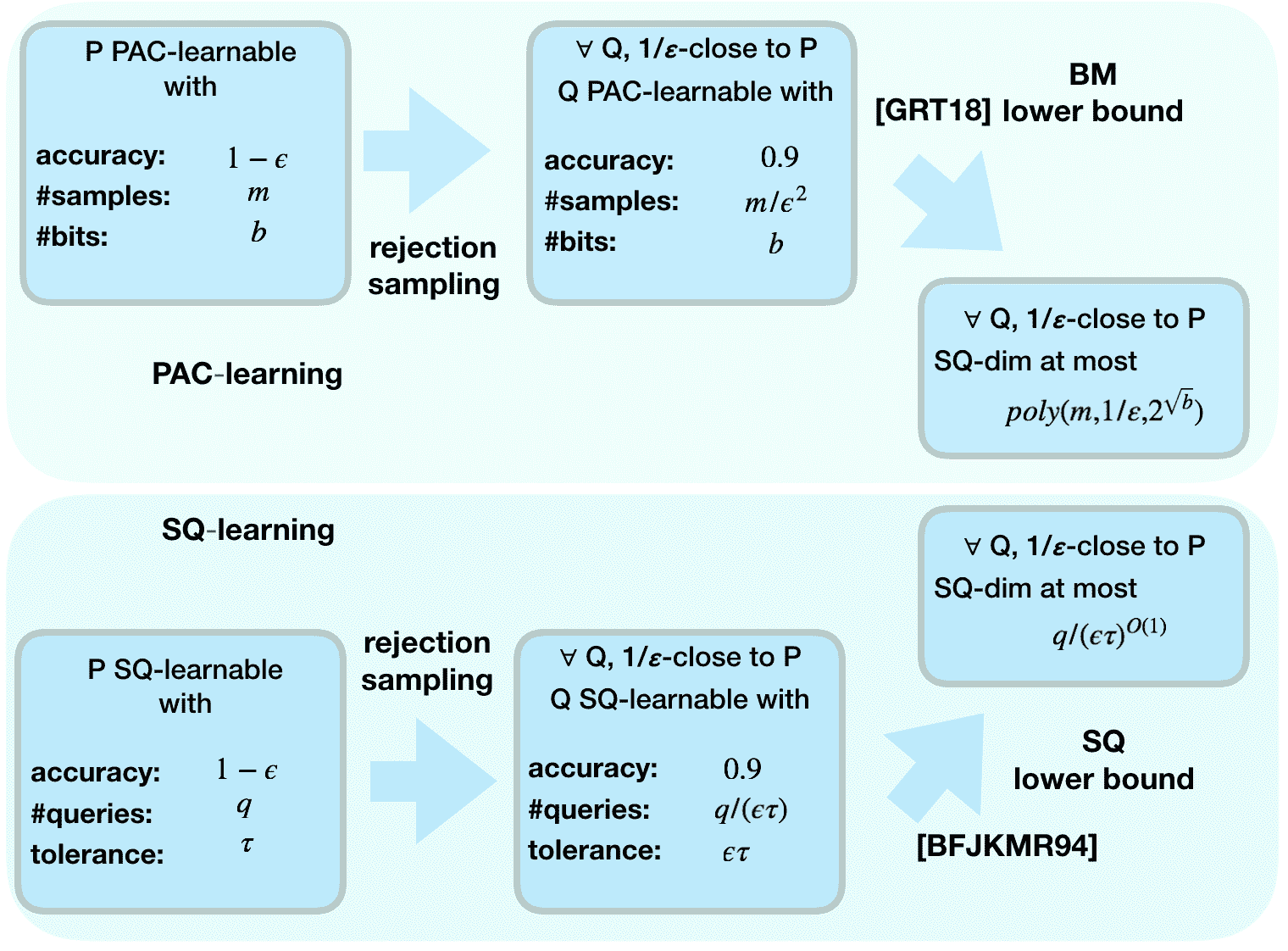} 
\caption{Proof outline (with asymptotic terms): from bounded SQ dimension under close distributions to strong learnability.}
%\scomment{can we add a brief explanation on the arrows? eg add "boosting" for the down arrow. I am not sure what to write for [SVW16]} \textcolor{blue}{Michal: We can add "general reduction from SQ to bounded memory". These pictures were made with keynote. Does it look professional enough? is there a better program to create the pictures?}
\label{fig:sq2bm}
\end{center}
\end{figure}

\paragraph{Overview.}
To prove \cref{thm:sq2bmPAC} we apply an extension of the Boosting-By-Majority (BBM) algorithm \cite{freund95} to the SQ framework due to \cite{aslam93}. Similarly to other popular boosting methods (e.g. AdaBoost \cite{freund97}), the algorithm operates by re-weighting the input sample and feeding the weak learner with sub-samples drawn according to the re-weighed distributions. The main challenge is to bound the SQ-dimension of the probability distributions maintained by the boosting algorithm. This will allow us to obtain a bound on the query complexity of the boosting process using \cref{prop:weakSQ} and thus conclude \cref{thm:sq2bmPAC}. Consequently, we deduce \cref{thm:sq2bmSQ} using  \cref{prop:steinhardt}. 

\paragraph{Reviewing Boost-By-Majority (BBM).}
Let $\mathcal{W}$ be a $\gamma$-weak learner with respect to the distribution $P$ with sample complexity $m_0$.
Similarly to most boosting algorithms, BBM operates by iteratively re-weighting and feeding a given $\gamma$-weak learner with $T$ i.i.d. samples $S_1,\ldots,S_T$ of size $m_0$.  
The outputs $h_1,\ldots,h_T$ of the weak learner are then aggregated into a majority vote classifier:
\[
h(x) = \texttt{Majority}(h_1(x),\ldots,h_T(x)) := \begin{cases} 1 & \sum_t h_t(x) > 0 \\ -1 & \textrm{otherwise} \end{cases}\;.
\]
To make the algorithm memory-efficient \cite{freund95} suggests to implement the re-weighting using \textit{rejection sampling}. Let $h_1,\ldots,h_t$ be the weak classifiers collected during the first $t$ rounds. At the beginning of round $t+1$, the algorithm draws an example $x \sim P$ and keeps it with probability
\begin{equation} \label{eq:bbmDist1}
w_{t+1}(x)=\mathrm{Binom} \left(T-t,\left\lfloor \frac{T-t-r(x)}{2}\right\rfloor ,1/2+\gamma \right)~\qquad \textrm{where}~~r(x):=\sum_{i=1}^t h_i(x).
\end{equation}
Therefore, the induced probability distribution on time $t$ is
\begin{equation} \label{eq:bbmDist2}
    P_{t+1}(x) = w_{t+1}(x) P(x)/Z
\end{equation}
where $Z$ is a normalization factor.
It repeats this step until either collecting
$m_0$ samples or rejecting $\Theta(\epsilon^{-3}\log T)$ consecutive examples. In the former scenario it feeds the weak learner with the resulted
sample, whereas in the latter scenario it aborts the boosting process and returns the hypothesis $h=\texttt{Majority}(h_1(x),\ldots,h_t(x))$.\footnote{In \cite{freund95}, the algorithm does not actually abort but proceeds by drawing random hypotheses for $T-t$ rounds. It was shown in \cite{jackson95}, Lemma 5.2, that (with the above rejection criteria) the algorithm can actually abort and return a majority vote.}
\begin{prop}\cite{freund95} \label{prop:bbm}
Let $\epsilon>0$. With probability at least $2/3$, the following hold:
\begin{enumerate}
\item BBM reaches an $\epsilon$-accurate hypothesis after at most  $T=O(\gamma^{-2}\log(1/\epsilon))$ rounds. 
\item There exists a global constant $C>0$ such that for every round $t$, the probability distribution $P_t$ satisfies $P_t(x) \le (C/ \epsilon^3) \cdot P(x)$ for all $x$.
\end{enumerate}
\end{prop}
\paragraph{SQ-Boost-By-Majority.}
Following \cite{aslam93} we describe how BBM can be carried out in the SQ model.  
Instead of having an access to a sampling oracle, the booster $\mathcal{A}$ is has an access to an SQ oracle with respect to the distribution $P$ and the target concept $c$. Similarly to BBM, the booster re-weights the points in $\mathcal{X}$ in iterative fashion, thereby defining a sequence of distributions,   $P_1,\ldots,P_T$. The weak learner $\mathcal{W}$ itself also works in the SQ model. That is, instead of requiring samples $S_1,\ldots,S_T$ drawn according to $P_1,\ldots,P_T$, it submits statistical queries to the boosting algorithm. The guarantee of the weak learner remains intact; provided that it gets sufficiently accurate answers (as determined by the tolerance parameter $\tau$), $\mathcal{W}$ should output a weak classifier whose correlation with the target concept is at least $\gamma$. 

Therefore, the challenging part in translating BBM to the SQ model is to enable simulating answers to statistical queries with respect to the distributions $P_1,\ldots,P_T$ given only an access to an SQ oracle with respect to the initial distribution $P$. Fortunately, the BBM's re-weighting scheme makes it rather easy. It follows from the definition of the distributions maintained by BBM (see \cref{eq:bbmDist1} and \cref{eq:bbmDist2})  
%\scomment{we didn't. do we need this? isn't the bound $P_t(x) \le C/\epsilon^3 P(x)$ enough?} 
%\textcolor{orange}{Alon: i added an explanation. is it clear now?} 
that in the beginning of round $t$, the space $\mathcal{X}$ partitions into $t$ regions such that the probability of points in each region is proportional to their initial distribution according to $P$. This allows simulating an \textit{exact} SQ query with respect to $P_{t+1}$ using $O(t)$ exact SQ queries to $P$. Furthermore, as shown in \cite{aslam93}, the fact that $P_t(x)\le (C/\epsilon^{3}) \cdot P(x)$ allows us to perform this simulation with suitable tolerance parameters. This is summarized in the next proposition. 
\begin{prop}[\cite{aslam93}] \label{prop:aslam}
Any statistical query with respect to the distribution $P_t$ with tolerance $\tau$ can be simulated using $O(t)$ statistical queries with respect to the original distribution $P$ with tolerance parameter $\Omega(\tau \cdot \mathrm{poly}(\epsilon))$.
\end{prop}

\paragraph{Upper bounding the SQ-dimension of SQ-BBM's distributions.}
In this part we derive an upper bound on the SQ-dimension of the
distribution $P_{1},\ldots,P_{T}$ maintained by SQ-BBM. To this end
we use our assumption that for all $Q\in\mathcal{P}_{\mu}(P)$,
$\mathrm{SQ}_Q(\mathcal{C})\le d$ where $\mu= \max \{C/\epsilon^{3}, 4d\}$.
While we cannot make sure that the distributions $P_{1},\ldots,P_{T}$
belong to $\mathcal{P_{\mu}}(P)$, we will still be able to derive
an upper bound on their SQ-dimension.
\begin{lem} \label{lem:upperSQdimBBM}
Let $P_{1},\ldots,P_{T}$ be the distributions maintained by SQ-BBM.
For every $t=1,\ldots T$, $\mathrm{SQ}_{P_{t}}(\mathcal{C})\le4d$.
\end{lem}

\paragraph{Putting it all together.}
We now complete the proofs of \cref{thm:sq2bmSQ} and \cref{thm:sq2bmPAC}.

\begin{proof}[Proof of \cref{thm:sq2bmSQ}]
From \cref{prop:weakSQ} we conclude that for any $Q \in \mathcal{P}_{\mu}(P)$ there exists a $(1/d)$-weak learner with query complexity $d$ and tolerance $1/(3d)$. Using this weak learner we apply SQ-BBM as described above. From \cref{lem:upperSQdimBBM} we know that for every distribution $P_t$ maintained by SQ-BBM, $\mathrm{SQ}_{P_t}(\mathcal{C})=O(d)$. Combining \cref{prop:bbm} and \cref{prop:aslam} we conclude that SQ-BBM reaches a $1-\epsilon$ accurate prediction after $T=O(\mathrm{poly}(d)\log(1/\epsilon))$ iterations while using at most $\mathrm{poly}(d/\epsilon)$ statistical queries with tolerance at least $\mathrm{poly}(\epsilon/d)$. 
\end{proof}

\begin{proof}[Proof of \cref{thm:sq2bmPAC}]
 \cref{prop:steinhardt} tells us that if a class $\mathcal{C}$
can be learned using $q$ statistical queries with tolerance $\tau$, then there is a PAC algorithm that learns $\mathcal{C}$ using
$m=O(\frac{q\log|\mathcal{C}|}{\tau^{2}}(\log(q)+\log\log(|\mathcal{C}|))$
samples and $b=O(\log|\mathcal{C}|\cdot\log(\frac{q}{\tau}))$ bits. \cref{thm:sq2bmSQ} gives an SQ learning algorithm $q = \mathrm{poly}(d/\epsilon)$ and $\tau \ge \mathrm{poly}(\epsilon/d)$, which gives a bounded memory learning algorithm with $m=\text{poly}(d/\epsilon) \cdot \log|\mathcal{C}| \cdot \log\log |\mathcal{C}|$ samples and $b=O(\log |\mathcal{C}| \cdot \log(d/\epsilon))$ bits.

\end{proof}

% Specifically, if $C$ is bounded-memory learnable, then number of
% samples is at most $|\mathcal{C}|^{o(1)}$ and number of bits is at
% most $o(\log^{2}|\mathcal{C}|).$ Plugging this parameters with $\epsilon=\vert\mathcal{C}\vert^{-o(1)}$ into the
% previous corollary we get the following. %Thus, for any $\epsilon$-close distribution $Q$, its SQ-dimension is bounded by $|\mathcal{H}|^{o(1)}\mathtt{poly}(\frac1\epsilon).$ We summarize it in the following corollary. 

% \begin{cor}
% Fix a class $\mathcal{C}$, distribution $P$ and $\epsilon=\vert\mathcal{C}\vert^{-o(1)}$. If for any $Q\in\mathcal{P}_{\epsilon}(P)$,
% $SQ_{Q}(\mathcal{C})\leq|\mathcal{C}|^{o(1)}$
% then $C$ is learned under $P$ with accuracy $1-\epsilon$, number
% of samples $|\mathcal{C}|^{o(1)}$,
% and $o(\log|C|\log|\mathcal{X}|).$ 
% \end{cor}

%% file: BMtoSQ.tex
\section{From bounded memory learning to bounded SQ dimension}\label{sec:bounded_memory_implies_bounded_sq}

%\textcolor{blue}{
%Question: 
%\begin{enumerate}
%\item does the notion of improper and proper is defined before? 
%\item in the preliminaries: mention why we don't care about confidence (can do bounded memory amplification). 
%\item To choose a good hypothesis out of $k$ hypotheses, takes $O(\log k)$ samples 
%\end{enumerate}
%}

%In this section we will show that if a class $\mathcal{C}$ is strongly learnable with bounded memory under distribution $P$, then for any $\epsilon$-close distribution $Q$ the statistical query dimension, $SQ_Q(\mathcal{C})$, is low. %, i.e., $SQ_{Q}(\mathcal{H})..$. The proof contains an intermediate step that shows how to weakly learn $\mathcal{C}$ under $Q$. A similar claim is also proven when learning is done in the statistical query framework rather than the PAC framework. 

In this section we prove our lower bounds: \cref{thm:bm2sqPAC} and \cref{thm:bm2sqSQ}. A schematic illustration of the proof is given in  \cref{fig:strong_learning_bounded_SQ}. 

\begin{figure}%[H]
\begin{center}
\includegraphics[scale=0.47]{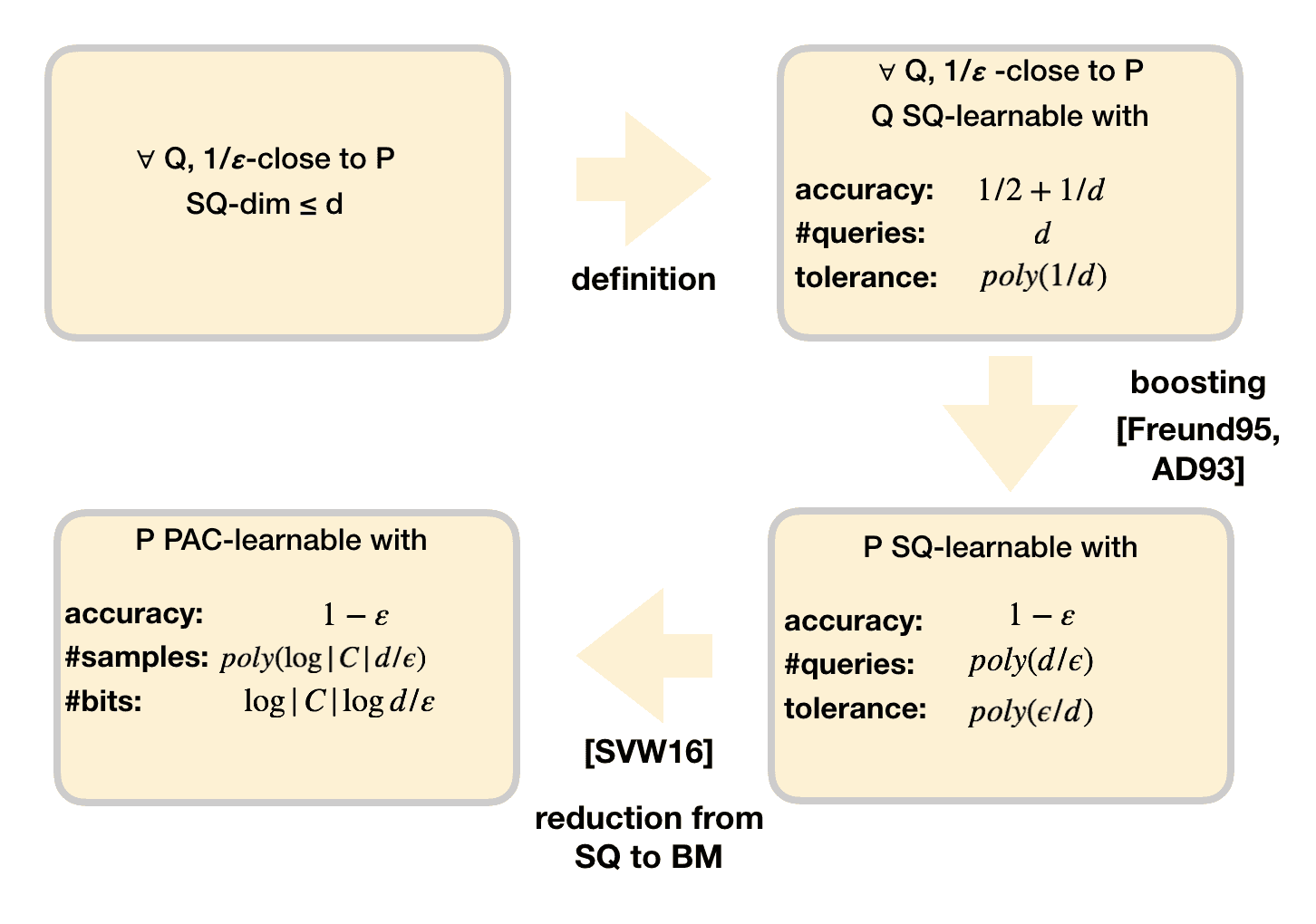}
\caption{Proof outline (with asymptotic terms): from strong learnability to bounded SQ-dimension under close distributions in PAC and SQ models.}
\label{fig:strong_learning_bounded_SQ}
\end{center}
\end{figure}

\paragraph{Overview.} %A recent result by \cite{garg19}  establishes an upper bound on $\mathrm{SQ}_Q(\mathcal{C})$ given memory-efficient learning that returns the \emph{exact} answer. To utilize this result we first translate strong learning guarantees with respect to $P$ into weak learning guarantees with respect to any close distribution $Q$. The problem with the resulted learner is that it is not exact and even not necessarily proper (i.e., it might return a hypothesis $h \notin \mathcal{C}$). We first show how to transform the weak learner into a proper learner. We ensure that the exact target concept $c$ is returned, as Large SQ-dimension implies  the hypotheses are far a part. Thus, proper learning enforces exact learning. 
%\textcolor{orange}{Alon: edited }

We use the rejection sampling technique to transform a given 
strong learner with respect to distribution $P$ into a weak learner with respect to any close distribution $Q$. 
%$_{\epsilon}(\mathcal{P})$ 
This can be established both in the PAC learning framework and the SQ framework. By virtue of \cref{prop:SQ_lower_bound}, this implies \cref{thm:bm2sqSQ}. To prove \cref{thm:bm2sqPAC}, we would like to use a recent result by \cite{garg19} that establishes an upper bound on $\mathrm{SQ}_Q(\mathcal{C})$ given memory-efficient learner. Unfortunately, the derivation in \cite{garg19} requires the learner to return the \emph{exact} target concept. Our weak learner does not necessarily satisfy this requirement. In fact, it is even not necessarily proper, i.e., it might return a hypothesis $h \notin \mathcal{C}$. To get around this obstacle, we first show how to transform any improper weak learning rule into a proper learning rule. Then, we focus on the hypotheses $\mathcal{H}\subseteq\mathcal{C}$ that constituents that SQ dimension, i.e, $SQ_Q(\mathcal{H})=SQ_Q(\mathcal{C})$. We ensure that the exact target concept $c$ is returned, as large $SQ_Q(\mathcal{H})$ implies that all hypotheses in $\mathcal{H}$ are far a part. %Thus, proper learning enforces exact learning. 

%any such distribution using a strong learner with respect to th rejection sampling technique to obtain a weak bounded-memory learner under any close distribution $Q$. This can be established both in the PAC learning framework and the SQ framework. To establish the claimed bounds on the SQ-dimension we use a re due to \cite{garg19} show a lower bound   goal is to bound $SQ_Q(C)$. We would have like to use known lower bounds that show that if there is a  bounded-memory learner then $SQ_Q(C)$ is low. The problem is that these bounds hold for proper learners, \cite{garg19} even requires exact learning. To overcome the gap between the improper algorithm described in Section~\ref{sec:bounded_sq_implies_bounded_memory} and the exact learning needed in \cite{garg19}, we show, in Lemma~\ref{clm:proper_then_improper}, how to transform any bounded-memory improper learner into a proper learner, which in turn yields an exact learner for any class with large SQ dimension. 

\paragraph{From strong learning to weak learning of close distributions.}

The next claim shows that if a class is strongly learnable under distribution $P$, then it is weakly learnable under \emph{any} close distribution $Q$. The idea is to utilize the closeness assumption in order to perform rejection sampling from $Q$ to simulate sampling from $P$. \begin{lem}\label{clm:strong_to_weak_pac_rejection_sampling} Let $P$ be a distribution
over $\mathcal{X}$. Assume that the concept class $\mathcal{C}$
can be learned with accuracy $1-0.1\epsilon$, %condifence $1-\delta$, 
$m$ samples, and $b$
bits under distribution $P$. Then, any probability distribution
$Q$ that is ($1/\epsilon$)-close to $P$ can be learned with accuracy $0.9$, %confidence $1-\delta-e^{-m}$,  
$O(m/\epsilon^2)$ samples, and $b$ bits. 
\end{lem}

\paragraph{Rejection sampling algorithm in the SQ model.}
Analogously to  \cref{clm:strong_to_weak_pac_rejection_sampling}, we can show that also under the SQ framework, strong learning implies weak learning of close distributions. The proof uses the same rejection sampling technique as in \cref{clm:strong_to_weak_pac_rejection_sampling}. %and is given in \cref{app:omitted}.
%\scomment{explain why we moved the proof to the appendix, eg the proof is very similar to the previous one}\textcolor{blue}{Michal: fixed}

\begin{lem}\label{thm:bm2sq_rejection_sampling_sq} Let $P$ be a distribution
over $\mathcal{X}$. Assume that the concept class $\mathcal{C}$
can be learned with accuracy $1-0.1\epsilon$, $q$ queries and tolerance $\tau$ under distribution $P$. Then, any probability distribution
$Q$ that is ($1/\epsilon$)-close to $P$ can be SQ-learned with accuracy $0.9$ using $O(q/\epsilon\tau)$ queries with tolerance $\epsilon\tau/2$. %and confidence $1-e^{-1}$.  
\end{lem}

\paragraph{From weak learning to low SQ-dimension.}
The next few claims establish the fact that if a class $\mathcal{C}$ is learnable with bounded memory under distribution $Q$, then the statistical dimension $SQ_Q(\mathcal{C})$ is low. 
% Equivalently, if a class has a large SQ dimension, then it cannot be learned with bounded memory. \acomment{move this discussion to the related work section. Here we should cite only the most relevant result}
% Several recent works proved lower bounds under memory constraints \cite{shamir14,raz16,kol17,moshkovitz17,moshkovitz18,raz17,garg18,dagan18,beame18,sharan19,garg19,dagan19}.
% Many of these works showed that if a class is ``mixing'' \cite{krivelevich06}, then it cannot learned with bounded memory. 
% A class with large statistical dimension is in fact mixing \cite{blum94,thomason89,alon16}. 
\begin{prop}[Corollary 8 in \cite{garg18}]\label{clm:lower_bound_for_sq}
Let $\mathcal{H} = \{h_1,\ldots,h_d\}$ be a class and $Q$ a distribution such that $SQ_Q(\mathcal{H})=d$. Then any learning algorithm that uses $m$ samples, $b$ bits and returns the exact correct hypothesis with probability at least $\Omega(1/m)$ must use at least $m=d^{\Omega(1)}$ samples or $b=\Omega(\log^2 d)$ bits.\footnote{In \cite{garg18} they consider the case where $Q$ is the uniform distribution. By creating a few copies of the examples in $\mathcal{X}$ we can transform a general \emph{known} distribution to be as close as to uniform as needed. Note that the size of the domain $\mathcal{X}$ is not a relevant parameter here.}
\end{prop}

The algorithm described in the previous section will not return the exact hypothesis, and more generally will not even be a proper learner (i.e., it will not necessarily return a hypothesis from the class). Fortunately, we can transform any improper learner into a proper learner without significantly increasing the neither the sample nor the space complexity. 
\begin{lem}\label{clm:proper_then_improper} Fix a class $\mathcal{C}$. Let $\mathcal{A}$ be an improper learning algorithm for $\mathcal{C}$
that uses $b$ bits, $m$ samples, and accuracy $1-\epsilon$.
Then there is an $(1-3\epsilon)$-accurate \emph{proper} learning algorithm
that uses $O(m)$ samples and $b+O(\log(|\mathcal{C}|/\epsilon))$ bits. %\scomment{why use "nicefrac" here? we didn't use it before}
\end{lem} 

\begin{lem}\label{thm:bm2sq_learning_than_small_sq}
Fix a class $\mathcal{C}$ and a distribution $Q$. %There exists a constant $M>0$ \scomment{I don't see where we use $M$ in the proof} \textcolor{blue}{Michal: in the application of Lemma~\ref{clm:lower_bound_for_sq}, added explanation} such that 
If $\mathcal{C}$ is learnable with accuracy $0.9$ under $Q$ using $m$ samples and $b$ bits,  then $$SQ_Q(\mathcal{C})\leq\max(m^{O(1)}, 2^{O(\sqrt{b})}).$$
\end{lem}

%  But we need to prove that all of the following hold 
% (1) with respect to SQ (sections,..) (2) approximated, the returned hypothesis can have $\epsilon$ error (4) improper learning is allowed 

% The next claim proves that improper learning is equivalent to proper
% learning when there are only memory constraints (and not computational
% constraints).

% The work \cite{garg18} proves that any proper learning algorithm that
% learns a class $\mathcal{C}$ with $SQ_{P}(\mathcal{C})\geq d$ either
% uses sample $m\geq d^{\Omega(1)}$ or $b\geq\Omega(\log^{2}d)$ bits.
% %$$SQ \geq d \Rightarrow m \geq d^{\Omega(1)} \text{ or } b \geq \Omega(\log^2 d)$$
% From the previous claim, we have that, equivalently, that if $m\leq d^{O(1)}\text{ and }b\leq O(\log^{2}d)\Rightarrow SQ_{P}(\mathcal{C})\leq d$. Or in other words, $$SQ_{P}(\mathcal{C})\leq \max(m^{O(1)}, 2^{O(\sqrt{b})})$$
% \textcolor{blue}{maybe add explanations?}

%\subsection{Summary}

\paragraph{Putting it all together.}
 %Putting the results of the last two subsections together we get the following corollary for PAC-learning.

We now complete the proofs of \cref{thm:bm2sqPAC} and \cref{thm:bm2sqSQ}.

\begin{proof}[Proof of \cref{thm:bm2sqPAC}]
Assume that the concept class $\mathcal{C}$
can be learned with accuracy $1-0.1\epsilon$, 
$m$ samples, and $b$
bits under distribution $P$. \cref{clm:strong_to_weak_pac_rejection_sampling} states that any probability distribution
$Q$ that is ($1/\epsilon$)-close to $P$ can be learned with accuracy $0.9$, %confidence $1-\delta-e^{-m}$,  
$O(m/\epsilon^2)$ samples, and $b$ bits. 
\cref{thm:bm2sq_learning_than_small_sq} completes the claim. 
\end{proof}

\begin{proof}[Proof of \cref{thm:bm2sqSQ}]
Assume that the concept class $\mathcal{C}$
can be learned with accuracy $1-0.1\epsilon$, $q$ queries and tolerance $\tau$ under distribution $P$. \cref{thm:bm2sq_rejection_sampling_sq} states that any probability distribution
$Q \in \mathcal{P}_{1/\epsilon}(P)$
can be SQ-learned with accuracy $0.9$,  $O(m/\epsilon\tau)$ queries, and tolerance $\epsilon\tau/2$. %and confidence $1-e^{-1}$.  
\cref{prop:SQ_lower_bound} completes the claim.
\end{proof}

%% file: appendix.tex
\section{Omitted Proofs} \label{app:omitted}

\begin{proof}[Proof of \cref{lem:upperSQdimBBM}]
Let $\delta=1/\mu$.
Consider the mixed distribution $\tilde{P}_{t}=\delta P+(1-\delta)P_{t}$. 
\cref{prop:bbm} implies that for all $x$, $P_{t}(x)\le\mu P(x)$. It
follows that 
\[
(\forall x)\qquad\tilde{P}_{t}(x)\le\delta P(x)+(1-\delta) \mu P(x)\le \mu P(x).
\]
Also, clearly we have that
\[
(\forall x)\qquad\tilde{P}_{t}(x)\ge\delta P(x) = \mu^{-1} P(x).
\]
 Hence, $\tilde{P}_{t}\in\mathcal{P}_{\mu}(P)$, and by our assumption we have
$\mathrm{SQ}_{\tilde{P}_{t}}(\mathcal{C}) \le d$. 

Assume by contradiction that there exist $m\ge4d$ hypotheses  $h_{1},\ldots,h_{m}\in\mathcal{C}$
such that 
\[
|\langle h_{i},h_{j}\rangle_{P_{t}}|\le1/m\qquad(\forall i\neq j\in[m]).
\]
 Therefore, for all $i\neq j\in[m]$, 
\[
\left|\langle h_{i},h_{j}\rangle_{\tilde{P}_{t}}\right|=\left|\delta\langle h_{i},h_{j}\rangle_{P}+(1-\delta)\langle h_{i},h_{j}\rangle_{P_{t}}\right|\le\delta+(1-\delta)\frac{1}{m}\le\frac{1}{4d}+\frac{1}{4d}=\frac{1}{2d}.
\]
In particular, it follows that $|\langle h_{i},h_{j}\rangle_{\tilde{P}_{t}}|\le\frac{1}{2d}$
for all $i\neq j\in[2d]$. This contradicts the fact that $\mathrm{SQ}_{\tilde{P}_{t}}(\mathcal{C})\le d$.
\end{proof}

\begin{proof}[Proof of \cref{clm:strong_to_weak_pac_rejection_sampling}]
Fix a distribution $P$, a class $\mathcal{C}$ and assume that there is an algorithm $\mathcal{A}$ that learns $\mathcal{C}$ under $P$ with accuracy $1-0.1\epsilon$, $m$ samples, and $b$ bits. We want to show that for any ($1/\epsilon$)-close distribution  $Q\in\mathcal{P}_{1/\epsilon}(P)$ there is an algorithm that learns the class $\mathcal{C}$ under distribution $Q$ with accuracy $0.9$, $O(m/\epsilon^2)$ samples, and $b$ bits.

At a high level, our analysis involves two steps. First, given a close distribution $Q$ we apply the rejection sampling technique to simulate sampling from the original distribution $P$. This enables us to run the algorithm $\mathcal{A}$. Then we translate the accuracy guarantee of $\mathcal{A}$ with respect to $P$ into a an accuracy guarantee with respect to $Q$. 
\paragraph{Rejection sampling.}
In \cref{alg:rejectionSampling} we detail the rejection sampling step mentioned above.

\begin{algorithm}
    \caption{Learning from examples distributed by $Q$}
    \label{alg:rejectionSampling}
\begin{algorithmic}[1]
\STATE Get a labeled example $x$ from $Q$.
\STATE Accept $x$ with probability $\frac{P(x)}{Q(x)}\epsilon$.
\STATE Call algorithm $\mathcal{A}$ with the accepted examples.
\end{algorithmic}
\end{algorithm}
We first note that the rejection sampling is well defined. Namely, by the closeness assumption, $\frac{P(x)}{Q(x)}\epsilon \in [0,1]$. The distribution induced by the rejection sampling is proportional to $P $ since
\[
Q(x)\cdot \frac{P(x)}{Q(x)}\epsilon= P(x)\epsilon.
\]

\paragraph{Strong learning with respect to $P \Rightarrow$ weak learning with respect to $Q$.}
By our assumption on $\mathcal{A}$, with probability at least $2/3$, it outputs
a hypothesis $h$ with accuracy at least $1-0.1\epsilon$. We next prove that $h$ forms a weak classifier with respect to $Q$. Denoting the target hypothesis by $c\in\mathcal{C}$, we have that 
$$L_{Q,c}(h)=\sum_{x: h(x)\neq c(x)} Q(x)\leq \sum_{x: h(x)\neq c(x)} \frac{1}{\epsilon}\cdot P(x)=\frac{1}{\epsilon}\cdot L_{P,c}(h)\leq \frac{1}{\epsilon}\cdot 0.1\epsilon = 0.1\;.$$
Thus, the accuracy is at least $0.9$.

So far we proved that the we indeed designed a learning algorithm for $Q$. Let's analyze the parameters of the algorithm. 
The rejection sampling technique does not require additional bits, thus number of bits is the same as number of bits used in $\mathcal{A}$.
We next bound the number of samples needed. 

We first note that the probability to accept an example $x$ is  $\frac{P(x)}{Q(x)}\epsilon\geq \epsilon^2$, as $Q$ is ($1/\epsilon$)-close to $P$. 
From Hoeffding's inequality, we know that if we get at least $2m/\epsilon^2$ samples, then the probability that the algorithm does not accept at least $m$ samples is smaller than $e^{-m}$.
Thus, with probability at least $1-me^{-m}$, the number of samples used by the new algorithm is $O(m/\epsilon^2)$.%\scomment{how can we talk on probability and expectation? either use one or the other.}\textcolor{blue}{Michal: fixed}

The confidence of the algorithm is at least $2/3- e^{-m}\cdot m\geq 7/12$ for large enough $m$. Standard amplification techniques can be used to ensure that the probability error is at most $2/3$, while increasing the sample complexity by at most a constant multiplicative factor.
\end{proof}

\begin{proof}[Proof of \cref{thm:bm2sq_rejection_sampling_sq}]
Fix a distribution $P$, a class $\mathcal{C}$ and assume that there is an algorithm $\mathcal{A}$ that learns $\mathcal{C}$ under $P$ with accuracy $1-0.1\epsilon$, $m$ queries, and tolerance $\tau$.
Denote the correct hypothesis by $c\in\mathcal{C}$. We want to show that for any $(1/\epsilon)$-close distribution  $Q\in\mathcal{P}_{1/\epsilon}(P)$ there is an algorithm that weakly learns the class $C$ under distribution $Q$ in the SQ framework. 

Fix a query $\psi$ that is used by $\mathcal{A}$. Ideally, we would like to replace it with a query $\psi'$ of the form 
$$
\psi'(x)=
\begin{cases}
\frac{P(x)}{Q(x)}\psi(x) \quad \text{ if } Q(x)\neq 0\\
0\quad\quad\quad\quad\quad otherwise
\end{cases}\;,
$$
%if $P(x) = 0$ (this implies that also $Q(x)=0$)  then $\psi'(x)=0$. 
since querying $\psi$ under $P$ is the same as querying $\psi'$ under $Q$, as $\mathbb{E}_Q[\psi'(x)c(x)]=\mathbb{E}_P[\psi(x)c(x)]$. The problem is that the range of $\psi'$ is not $\{-1,1\}$. 
To fix it, we will replace $\psi$ with several queries $\psi_1,\ldots,\psi_n$ that their range is $\{-1,1\}$ and their average, $\frac{1}{n}\sum_{i=1}^n\psi_i,$   approximately returns the correct query, i.e., $\psi'\approx\frac{1}{n}\sum_{i=1}^n\psi_i$.

For every $x\in\mathcal{X}$ we would like to use \cref{lem:bm2sq_rejection_sampling_sq} below in order to define $\psi_i(x)$. The first step will be to make sure that $\psi'(x)$ is in $[-1,1]$. To achieve that we focus on $\epsilon\psi'(x),$ because it is equal to $\epsilon\frac{P(x)}{Q(x)}\psi(x)$ and $$0<\epsilon\cdot\frac{P(x)}{Q(x)}\leq\epsilon\cdot\frac1\epsilon=1.$$
Using \cref{lem:bm2sq_rejection_sampling_sq}, there are $n=O(1/\epsilon\tau)$ queries $\psi_i$ such that for every $x\in\mathcal{X}$ it holds that $$\left\vert\frac1n\sum_{i=1}^n \psi_i(x)-\epsilon\psi'(x)\right\vert\leq\frac{\epsilon\tau}2.$$
From this we can deduce that 
$$\bigg\vert\frac{1}{\epsilon}\cdot\frac{1}{n}\sum_{i=1}^n\mathbb{E}_Q[\psi_i(x)c(x)]-\mathbb{E}_P[\psi(x)c(x)]\bigg\vert\leq\frac\tau2.$$ 

%We just need to make sure that for every $x\in\mathcal{X}$ $\sum_{i=1}^n\psi_i(x)=\psi'(x)$ and we will be over. 

%We first define a random query $\psi^r$, i.e., for each $x$ with probability $p_x$, the value $\psi^r(x)=1$ and with probability $1-p_x$ the value $\psi^r(x)=-1$. We would like that $\mathbb{E}_Q[\psi^r]=\mathbb{E}_P[\psi],$ where the first expectation also includes the randomness in $\psi^r$. To achieve that we define $$p_x:=\frac{\epsilon\frac{P(X)}{Q(X)}\psi(x)+1}{2}.$$ Since $Q$ is $\epsilon$-close to $P$, it holds that  $\epsilon\frac{P(X)}{Q(X)}\psi(x)\in[-1,1]$. Thus, $p_x\in[0,1]$, as desired. Importantly, $\mathbb{E}_Q[\psi^r]=\epsilon\mathbb{E}_P[\psi]$. As desired. 

To summarize, the new learning algorithm $\mathcal{A}'$ that learns under distribution $Q$ will simulate algorithm $\mathcal{A}$ and whenever a query $\psi$ will be needed, it will take $O(1/\epsilon\tau)$  queries created by \cref{lem:bm2sq_rejection_sampling_sq} and return their average times $1/\epsilon$. Thus, $\mathcal{A}$ uses $O(m/\epsilon\tau)$ queries and its tolerance is $\epsilon\tau/2$.
%The  tolerance of $\mathcal{A}'$ is $2\tau$.

%Using Hoeffding's inequality, we know that that the probability the result will be far from $\mathbb{E}_P[\psi]$ by $\tau$ is at most $2e^{-1}$. 

% \textcolor{blue}{add pseudo-code?}
% \begin{algorithm}
%     \caption{SQ Learning using queries with $Q$}
% \begin{algorithmic}[1]
% \STATE Input: algorithm $\mathcal{A}$ to SQ-learn with $P$
% \STATE For each query in $\mathcal{A}$ 
% \STATE take 
% \STATE Accept $x$ with probability $\frac{P(x)}{Q(x)}\epsilon$
% \STATE Use algorithm $\mathcal{A}$ to learn $P$ with the accepted examples
% \end{algorithmic}
% \end{algorithm}

%is correct because the distribution is $P$. The same as before, a good distribution according to $P$ is a good distribution according to $Q$.

\end{proof}

\begin{proof}[Proof of \cref{clm:proper_then_improper}]
Fix a class $\mathcal{C}$ and an improper learning algorithm $\mathcal{A}$ for $\mathcal{C}$. Denote the number of bits it uses by $b$, the number of samples by $m$, and the accuracy by $1-\epsilon$.
Define the algorithm $\mathcal{A}'$ as follows:
\begin{enumerate}
\item Run algorithm $\mathcal{A}$ that outputs hypothesis $h$ as its answer.
\item Go over all hypothesis in $\mathcal{C}$ and return one that agrees with $h$
on $1-2\epsilon$ of the examples by testing consistency on $O(\log\vert\mathcal{C}\vert/\epsilon^2)$ random examples. 
\end{enumerate}
% \begin{itemize}
%     \item run $\mathcal{A}$
%     \item go over all hypothesis in $\mathcal{H}$ and return one that agrees on $1-\epsilon$ of the examples 
% \end{itemize}
Note that the second step does not use new samples and requires only
$\log|\mathcal{C}|+O(\log(\log|\mathcal{C}|/\epsilon))=O(\log (|\mathcal{C}|/\epsilon))$ additional bits. The algorithm $\mathcal{A}'$
functions correctly, because by the definition of the algorithm $\mathcal{A}$
there must be hypothesis in $\mathcal{C}$ that agrees on $(1-\epsilon)$ of the examples. By Hoeffding's inequality, the probability that there is a hypothesis that deviates by more than $\epsilon$ in approximating its loss is small and standard amplification techniques can be used to ensure that the probability error is at most $2/3$, while increasing the sample complexity by at most a constant multiplicative factor. %\scomment{then we should have $O(m)$ and not $m$ in the lemma statement}\textcolor{blue}{Michal: you are right, fixed}.
The accuracy of $\mathcal{A}'$ is at least $1-3\epsilon$. 
\end{proof}

\begin{proof}[Proof of \cref{thm:bm2sq_learning_than_small_sq}]
Fix a class $\mathcal{C}$ and a distribution $Q$. Assume $\mathcal{C}$ is learnable under $Q$ with $m$ samples, $b$ bits, and accuracy $0.9$. Assume also that $SQ_Q(\mathcal{C})=d$. Thus, there are $d$ hypotheses $\mathcal{H}=\{h_1,\ldots,h_d\}$ such that $\vert\langle h_{i},h_{j}\rangle_{Q}\vert\leq1/d$.  Since  $\mathcal{H}\subseteq \mathcal{C}$ and by our assumption on the learnability of $\mathcal{C}$, we get that $\mathcal{H}$ is learnable under $Q$ with $m$ samples, $b$ bits, and accuracy $0.9$.
From \cref{clm:proper_then_improper}, we get that $\mathcal{H}$ is \emph{properly} learnable under $Q$ with $O(m)$ samples, $b+O(\log |\mathcal{H}|)$ bits, and accuracy $0.7$. 

We can deduce that there is a learning algorithm for $\mathcal{H}$ that returns the exact hypothesis, as the hypotheses in $\mathcal{H}$ are far apart from each other. Specifically, we know that between any two hypotheses $i\neq j$ there is at least $\frac12-\frac1{2d}$ disagreement. If  $\frac12-\frac1{2d} > 0.3$, then learning exactly is equivalent to properly learning up to accuracy $0.7$. The equation $\frac12-\frac1{2d} > 0.3$ is equivalent to $d = \Omega(1)$.  

Since the hypotheses in $\mathcal{H}$ are far apart from each other, the number of bits $\mathcal{A}$ uses is lower bounded by $b\geq\log\vert\mathcal{H}\vert$, as the hypothesis in $\mathcal{H}$ returned by the algorithm must be computed from its internal state. Thus the memory requirement of the proper learning algorithm is $O(b)$ bits.

Now we can apply \cref{clm:lower_bound_for_sq}, as for large enough constant $M$, for $m\geq M$, the probability to 
succeed, $2/3$, is $\Omega(1/m)$. %We assume that $1-\delta>c/m$ 
We get that 
$m=d^{\Omega(1)}$  or $b=\Omega(\log^2 d)$. 
Equivalently, $d=m^{O(1)}$ or $d=2^{O(\sqrt{b})}$.
%$m\leq d^{O(1)}\text{ and }b+\log m+\log\vert\mathcal{X}\vert\leq O(\log^{2}d)$. 
In other words, $SQ_{Q}(\mathcal{C})\leq \max(m^{O(1)}, 2^{O(\sqrt{b})})$. 
\end{proof}

\begin{lem}\label{lem:bm2sq_rejection_sampling_sq}
For any $\gamma\in[-1,1]$ and $\tau\in(0,1]$, there are $n=O(1/\tau)$ numbers $y_1,\ldots,y_n\in\{-1,1\}$ such that $|\frac1n\sum_i y_i-\gamma|\leq\tau$. 
\end{lem}
%The proof is detailed in Appendix \ref{app:omitted}.

\begin{proof}
Take $n$ such that $1/n < \tau$. Let $k \in \{0,1,\ldots,n\}$ be such that $(n-2k)/n$ is $1/n$ close to $\gamma$. Take $y_1=\ldots=y_k=-1$ and $y_{k+1}=\ldots=y_n=1$.
\end{proof}

%% file: main.bbl
\begin{thebibliography}{10}

\bibitem{aslam93}
Javed~A Aslam and Scott~E Decatur.
\newblock General bounds on statistical query learning and {PAC} learning with
  noise via hypothesis boosting.
\newblock In {\em Proceedings of 1993 IEEE 34th Annual Foundations of Computer
  Science}, pages 282--291. IEEE, 1993.

\bibitem{balcazar07}
Jos{\'e}~L Balc{\'a}zar, Jorge Castro, David Guijarro, Johannes K{\"o}bler, and
  Wolfgang Lindner.
\newblock A general dimension for query learning.
\newblock {\em Journal of Computer and System Sciences}, 73(6):924--940, 2007.

\bibitem{beame18}
Paul Beame, Shayan~Oveis Gharan, and Xin Yang.
\newblock Time-space tradeoffs for learning finite functions from random
  evaluations, with applications to polynomials.
\newblock In {\em Conference On Learning Theory}, pages 843--856, 2018.

\bibitem{ben08}
Shai Ben-David, Tyler Lu, and D{\'a}vid P{\'a}l.
\newblock Does unlabeled data provably help? worst-case analysis of the sample
  complexity of semi-supervised learning.
\newblock In {\em COLT}, pages 33--44, 2008.

\bibitem{benedek91}
Gyora~M Benedek and Alon Itai.
\newblock Learnability with respect to fixed distributions.
\newblock {\em Theoretical Computer Science}, 86(2):377--389, 1991.

\bibitem{blum94}
Avrim Blum, Merrick Furst, Jeffrey Jackson, Michael Kearns, Yishay Mansour, and
  Steven Rudich.
\newblock Weakly learning {DNF} and characterizing statistical query learning
  using fourier analysis.
\newblock In {\em STOC}, volume~94, pages 253--262, 1994.

\bibitem{Blum2015}
Avrim Blum, John Hopcroft, and Ravindran Kannan.
\newblock Foundations of data science.
\newblock {\em Vorabversion eines Lehrbuchs}, 2016.

\bibitem{blumer89}
Anselm Blumer, Andrzej Ehrenfeucht, David Haussler, and Manfred~K Warmuth.
\newblock Learnability and the vapnik-chervonenkis dimension.
\newblock {\em Journal of the ACM (JACM)}, 36(4):929--965, 1989.

\bibitem{bshouty02}
Nader~H Bshouty and Dmitry Gavinsky.
\newblock On boosting with polynomially bounded distributions.
\newblock {\em Journal of Machine Learning Research}, 3(Nov):483--506, 2002.

\bibitem{dagan19}
Yuval Dagan, Gil Kur, and Ohad Shamir.
\newblock Space lower bounds for linear prediction in the streaming model.
\newblock In {\em Conference on Learning Theory}, pages 929--954, 2019.

\bibitem{dagan18}
Yuval Dagan and Ohad Shamir.
\newblock Detecting correlations with little memory and communication.
\newblock In {\em Conference On Learning Theory}, pages 1145--1198, 2018.

\bibitem{feldman12}
Vitaly Feldman.
\newblock A complete characterization of statistical query learning with
  applications to evolvability.
\newblock {\em Journal of Computer and System Sciences}, 78(5):1444--1459,
  2012.

\bibitem{freund92}
Yoav Freund.
\newblock An improved boosting algorithm and its implications on learning
  complexity.
\newblock In {\em Proceedings of the fifth annual workshop on Computational
  learning theory}, pages 391--398. ACM, 1992.

\bibitem{freund95}
Yoav Freund.
\newblock Boosting a weak learning algorithm by majority.
\newblock {\em Information and computation}, 121(2):256--285, 1995.

\bibitem{freund97}
Yoav Freund and Robert~E Schapire.
\newblock A decision-theoretic generalization of on-line learning and an
  application to boosting.
\newblock {\em Journal of computer and system sciences}, 55(1):119--139, 1997.

\bibitem{garg18}
Sumegha Garg, Ran Raz, and Avishay Tal.
\newblock Extractor-based time-space lower bounds for learning.
\newblock In {\em Proceedings of the 50th Annual ACM SIGACT Symposium on Theory
  of Computing}, pages 990--1002. ACM, 2018.

\bibitem{garg19}
Sumegha Garg, Ran Raz, and Avishay Tal.
\newblock Time-space lower bounds for two-pass learning.
\newblock In {\em 34th Computational Complexity Conference (CCC 2019)}. Schloss
  Dagstuhl-Leibniz-Zentrum fuer Informatik, 2019.

\bibitem{impagliazzo95}
Russell Impagliazzo.
\newblock Hard-core distributions for somewhat hard problems.
\newblock In {\em Proceedings of IEEE 36th Annual Foundations of Computer
  Science}, pages 538--545. IEEE, 1995.

\bibitem{jackson95}
Jeffrey~C Jackson.
\newblock The harmonic sieve: A novel application of fourier analysis to
  machine learning theory and practice.
\newblock Technical report, Carnegie Mellon University Pittsburgh School Of
  Computer Science, 1995.

\bibitem{kearns98}
Michael Kearns.
\newblock Efficient noise-tolerant learning from statistical queries.
\newblock {\em Journal of the ACM (JACM)}, 45(6):983--1006, 1998.

\bibitem{klivans99}
Adam~R Klivans and Rocco~A Servedio.
\newblock Boosting and hard-core sets.
\newblock In {\em 40th Annual Symposium on Foundations of Computer Science
  (Cat. No. 99CB37039)}, pages 624--633. IEEE, 1999.

\bibitem{kol17}
Gillat Kol, Ran Raz, and Avishay Tal.
\newblock Time-space hardness of learning sparse parities.
\newblock In {\em Proc. 49th ACM Symp. on Theory of Computing}, 2017.

\bibitem{littlestone88}
Nick Littlestone.
\newblock Learning quickly when irrelevant attributes abound: A new
  linear-threshold algorithm.
\newblock {\em Machine learning}, 2(4):285--318, 1988.

\bibitem{moshkovitz17}
Dana Moshkovitz and Michal Moshkovitz.
\newblock Mixing implies lower bounds for space bounded learning.
\newblock In {\em Conference on Learning Theory}, pages 1516--1566, 2017.

\bibitem{moshkovitz18}
Dana Moshkovitz and Michal Moshkovitz.
\newblock Entropy samplers and strong generic lower bounds for space bounded
  learning.
\newblock In {\em 9th Innovations in Theoretical Computer Science Conference
  (ITCS 2018)}. Schloss Dagstuhl-Leibniz-Zentrum fuer Informatik, 2018.

\bibitem{moshkovitz17general}
Michal Moshkovitz and Naftali Tishby.
\newblock A general memory-bounded learning algorithm.
\newblock {\em arXiv preprint arXiv:1712.03524}, 2017.

\bibitem{raz16}
Ran Raz.
\newblock Fast learning requires good memory: A time-space lower bound for
  parity learning.
\newblock In {\em Proc. 57th IEEE Symp. on Foundations of Computer Science},
  2016.

\bibitem{raz17}
Ran Raz.
\newblock A time-space lower bound for a large class of learning problems.
\newblock In {\em 2017 IEEE 58th Annual Symposium on Foundations of Computer
  Science (FOCS)}, pages 732--742. IEEE, 2017.

\bibitem{sabato13}
Sivan Sabato, Nathan Srebro, and Naftali Tishby.
\newblock Distribution-dependent sample complexity of large margin learning.
\newblock {\em The Journal of Machine Learning Research}, 14(1):2119--2149,
  2013.

\bibitem{schapire90}
Robert~E Schapire.
\newblock The strength of weak learnability.
\newblock {\em Machine learning}, 5(2):197--227, 1990.

\bibitem{schapire91}
Robert~E Schapire.
\newblock The design and analysis of efficient learning algorithms.
\newblock Technical report, Massachusetts Inst Of Tech Cambridge Lab For
  Computer Science, 1991.

\bibitem{shamir14}
O.~Shamir.
\newblock Fundamental limits of online and distributed algorithms for
  statistical learning and estimation.
\newblock In {\em Proceedings of the 27th International Conference on Neural
  Information Processing Systems}, NIPS'14, pages 163--171, 2014.

\bibitem{sharan19}
Vatsal Sharan, Aaron Sidford, and Gregory Valiant.
\newblock Memory-sample tradeoffs for linear regression with small error.
\newblock {\em arXiv preprint arXiv:1904.08544}, 2019.

\bibitem{simon07}
Hans~Ulrich Simon.
\newblock A characterization of strong learnability in the statistical query
  model.
\newblock In {\em Annual Symposium on Theoretical Aspects of Computer Science},
  pages 393--404. Springer, 2007.

\bibitem{steinhardt16}
Jacob Steinhardt, Gregory Valiant, and Stefan Wager.
\newblock Memory, communication, and statistical queries.
\newblock In {\em Conference on Learning Theory}, pages 1490--1516, 2016.

\bibitem{szorenyi09}
Bal{\'a}zs Sz{\"o}r{\'e}nyi.
\newblock Characterizing statistical query learning: simplified notions and
  proofs.
\newblock In {\em International Conference on Algorithmic Learning Theory},
  pages 186--200. Springer, 2009.

\bibitem{valiant84}
Leslie~G Valiant.
\newblock A theory of the learnable.
\newblock In {\em Proceedings of the sixteenth annual ACM symposium on Theory
  of computing}, pages 436--445. ACM, 1984.

\bibitem{vapnik15}
Vladimir~N Vapnik and A~Ya Chervonenkis.
\newblock On the uniform convergence of relative frequencies of events to their
  probabilities.
\newblock In {\em Measures of complexity}, pages 11--30. Springer, 2015.

\bibitem{vayatis99}
Nicolas Vayatis and Robert Azencott.
\newblock Distribution-dependent vapnik-chervonenkis bounds.
\newblock In {\em European Conference on Computational Learning Theory}, pages
  230--240. Springer, 1999.

\bibitem{yang01}
Ke~Yang.
\newblock On learning correlated boolean functions using statistical queries.
\newblock In {\em International Conference on Algorithmic Learning Theory},
  pages 59--76. Springer, 2001.

\bibitem{yang05}
Ke~Yang.
\newblock New lower bounds for statistical query learning.
\newblock {\em Journal of Computer and System Sciences}, 70(4):485--509, 2005.

\end{thebibliography}
